\documentclass{article}
\usepackage{float}
\usepackage{bm}
\usepackage{amsmath}
\usepackage{amsthm}
\usepackage{amssymb}
\usepackage{graphicx}
\usepackage{caption}
\usepackage{subcaption}
\usepackage{breqn}
\usepackage{xspace}
\usepackage{framed}
\usepackage[round,authoryear]{natbib}
\usepackage[title]{appendix}

\usepackage[unicode=true,
 bookmarks=false,
 breaklinks=false,pdfborder={0 0 1},backref=false,colorlinks=false]
 {hyperref}
\hypersetup{
 pdfstartview=FitH}

\makeatletter

\pdfpageheight\paperheight
\pdfpagewidth\paperwidth

\floatstyle{ruled}
\newfloat{algorithm}{tbp}{loa}
\providecommand{\algorithmname}{Algorithm}
\floatname{algorithm}{\protect\algorithmname}

 \theoremstyle{definition}
  
  \theoremstyle{definition}
  
  \theoremstyle{plain}
  \newtheorem{lem}{Lemma}
\theoremstyle{plain}
\newtheorem{thm}{Theorem}

\DeclareMathOperator*{\argmax}{arg\,max}
\DeclareMathOperator*{\argmin}{arg\,min}

\DeclareMathOperator{\diam}{diam}
\usepackage{algorithm,algpseudocode}

\newcommand{\x}{{\mathbf{x}}}

\makeatother

  \providecommand{\examplename}{Example}

\newcommand{\OFW}{{\textsf{Meta-Frank-Wolfe}}\xspace}
\newcommand{\OGA}{{\textsf{Online Gradient Ascent}}\xspace}
\newcommand{\SGA}{{\textsf{Surrogate Gradient Ascent}}\xspace}
\newcommand{\Rand}{{\textsf{Random100}}\xspace}
\usepackage[capitalise]{cleveref}

\crefname{lem}{Lemma}{Lemmas}
\crefname{thm}{Theorem}{Theorems}
\crefname{appsec}{Appendix}{Appendices}

\usepackage{color}

\usepackage{authblk}
\title{Online Continuous Submodular Maximization}

\author[1,2]{Lin Chen}
\author[3]{Hamed Hassani}
\author[1,2]{Amin Karbasi}
\affil[1]{Yale Institute for Network Science} 
\affil[2]{Department of Electrical Engineering, Yale University}
\affil[3]{Department of Electrical and Systems Engineering, University of 	
Pennsylvania}
\affil[ ]{\normalsize \texttt{\{lin.chen, amin.karbasi\}@yale.edu, 
hassani@seas.upenn.edu}}

\date{}
\usepackage[letterpaper, margin=1in]{geometry}

\begin{document}

\maketitle

\begin{abstract}
In this paper, we consider an online optimization process, where the 
objective functions are not convex (nor concave) but instead belong 
to a broad class of  continuous submodular functions. We first 
propose a variant of the Frank-Wolfe algorithm that has access to the 
full gradient of the objective functions. We show that it achieves a 
regret bound of  
$ O(\sqrt{T}) $  (where $ T $ is the horizon of the online 
optimization problem)  against a $(1-1/e)$-approximation to the best 
feasible solution in
hindsight.  However, in many scenarios, only an unbiased estimate of 
the gradients are available. For such settings, we then propose an 
online stochastic gradient ascent algorithm that also achieves a 
regret bound of  $ 
O(\sqrt{T}) $ regret, albeit against a weaker $1/2$-approximation to 
the best feasible solution in
hindsight. We also generalize our results to $ \gamma $-weakly submodular 
functions and prove the same sublinear regret bounds. Finally, we demonstrate 
the efficiency of our algorithms on a few problem instances, including  
non-convex/non-concave quadratic programs, multilinear extensions of submodular 
set functions, and D-optimal design. 
%
\end{abstract}

\section{Introduction}

In the past few years, the era of big data has necessitated scalable machine learning techniques that can process an unprecedentedly growing amount of data, including data generated by users (e.g., pictures, videos, and tweets), wearable devices (e.g., statistics of steps, walking and running distance) and monitoring sensors (e.g., satellite and traffic images). At the same time, it is practically impossible to lay out  an exact mathematical model for such data generating processes. Thus, any optimization techniques applied to the data should be robust against  imperfect and even fundamentally unavailable knowledge. 

A  robust approach to optimization (in the face of uncertainty) in many fields, 
including artificial intelligence, statistics, and machine learning,  is to 
look 
at the optimization itself as a process \citep{hazan2016introduction} that 
learns  from experience as more aspects of the problem are observed. 
This framework is formally known as \emph{online optimization} and is performed 
in a sequence of consecutive rounds. In each round, the learner/algorithm has 
to choose an action (from the set of feasible actions) and then  the 
environment/adversary reveals a reward function. The goal is then to minimize 
regret, a metric borrowed from game theory, that measures the difference 
between the accumulated reward received by the algorithm  and that of the best 
\emph{fixed} action in hindsight. When the objective functions are concave 
and the feasible set forms a convex body, the problem has been extensively 
studied in the machine learning community under the name of \emph{online 
convex optimization} (OCO). It is well known that any algorithm for OCO incurs 
$ \Omega(\sqrt{T})$ regret in the worst case ~\citep{hazan2016introduction}. 
There are also several algorithms  that match this lower bound such as online 
gradient descent (OGD)~\citep{zinkevich2003online} and 
regularized-follow-the-leader 
(RFTL)~\citep{abernethy2008efficient,shalev2007primal,shalev2007online}.

%
%
%
%

Even though optimizing convex/concave functions can be done efficiently, most 
problems in statistics and artificial intelligence are non-convex. Examples 
include training deep neural networks,  learning latent variables, non-negative 
matrix factorization, Bayesian inference, and clustering, among many others. As 
a result, there has been a burst of recent research to directly optimize such 
functions. Due to the fact that in general it is NP-hard to compute the global 
optimum of a non-convex function, most non-convex optimization algorithms focus 
on finding a local optimum. Naturally, for online non-convex optimization 
(ONCO) one needs to define an appropriate  notion of regret related to 
convergence to an (approximate) local optimum \citep{hazan2017efficient}. 

In this work, we consider a rich subclass of non-convex/non-concave reward functions called \emph{continuous submodular functions}
\citep{wolsey1982analysis,bach2015submodular,vondrak2007submodularity}.
It has been very recently established that in the offline setting, first order 
methods provide tight approximation 
guarantees~\citep{chekuri2015multiplicative,bian16guaranteed,hassani2017gradient}.
To the best of our knowledge,  our work is the first that systematically studies the online continuous submodular maximization problem and provides \emph{no-regret} guarantees along with developing efficient algorithms. 

%
%

\paragraph{Our contributions} In summary, for monotone and  
continuous (weakly) DR-submodular reward functions\footnote{A 
DR-submodular 
function is a function that is defined on a continuous domain and 
exhibits the diminishing returns property. We present its formal 
definition in~\cref{sec:submodularity}.}, and subject to a general 
convex body (not necessarily down-closed),  we propose two 
algorithms, both with sublinear regret bounds, depending on what 
side information is available regarding the gradients. 
\begin{itemize}
	\item When the gradients are available, we propose \OFW, a variant of a 
	Frank-Wolfe algorithm,  that achieves  a $(1-1/e)$ approximation factor of 
	the best fixed offline solution in hindsight up to an $ O(\sqrt{T}) $ 
	regret term, where $ T $ is the horizon of the online maximization problem. 
    \item When only  unbiased estimates of the gradients are available, we 
    propose \OGA,  that achieves  a $1/2$ approximation factor of the best 
    fixed offline solution in hindsight up to an $ O(\sqrt{T}) $ regret term. 
    \item More generally, for $ \gamma $-weakly DR-submodular functions, we 
    show that \OGA  yields a $ \frac{\gamma^2}{\gamma^2+1} $ approximation 
    guarantee to the best fixed offline solution in hindsight up to an $ 
    O(\sqrt{T}) $ regret term ($ \gamma=1 $ corresponds to a DR-submodular 
    function). 
%
\end{itemize}
%

\section{Preliminaries}
In this section, we precisely define the concepts that we will use throughout the paper. 

\subsection{Notation
}

\paragraph{Projection}
As we will discuss the projected (stochastic) gradient ascent later in \cref{sub:grad-ascent}, we introduce the notation of projection operator here, which is denoted by
\[ \Pi_{\mathcal{P} }(\mathbf{x}) \triangleq \argmin_{\mathbf{v} \in \mathcal{P} } \lVert \mathbf{x} - \mathbf{v} \rVert.
 \]
 Intuitively, the projection of point $ \x $ onto a convex set $ \mathcal{P} $ is a point in $ \mathcal{P} $ that is closest to $ \x $.
 
\paragraph{Radius and Diameter}

For any set of points $ S $, its radius $ \rho(S) $ is defined to be 
$ \sup_{\x \in S} \lVert \x \rVert $ while its diameter $ \diam(S) $ 
is defined to be $ \sup_{\mathbf{x},\mathbf{y}\in S} \lVert 
\mathbf{x}-\mathbf{y} \rVert $. By the triangle equality, we 
immediately have $ \diam(S) \leq 2\rho(S) $.

\paragraph{Smoothness}
To derive guarantees for the proposed algorithm, we will make the  assumption 
that the gradients of the objective functions satisfy the Lipschitz condition.
	A differentiable function $ f:\mathcal{X} \subseteq\mathbb{R}^n \to 
	\mathbb{R} $ is said to be \emph{$ \beta $-smooth} if for any $ 
	\mathbf{x},\mathbf{y} \in \mathcal{X} $, we have $ \lVert \nabla 
	f(\mathbf{x}) - \nabla f(\mathbf{y}) \rVert \leq \beta \lVert 
	\mathbf{x}-\mathbf{y} \rVert $.
\subsection{Submodularity}\label{sec:submodularity}
\paragraph{Submodular Functions on Lattices}
Suppose that $(L,\vee,\wedge)$ is a lattice\footnote{A lattice is a set $ L $ 
equipped with two commutative and associative binary operations $ \vee $ and $ 
\wedge $ connected by the absorption law, i.e., $ a\vee(a\wedge b)=a $ and $ 
a\wedge (a \vee b)=a $, $ \forall a,b\in L $~\citep{sankappanavar1981course}.}. 
A function $f:L\to\mathbb{R}$
is said to be \emph{submodular}~\citep{topkis1978minimizing} if $\forall x,y\in 
L$, we have
\[
f(x)+f(y)\geq f(x\vee y)+f(x\wedge y).
\]
Furthermore, a function $ f:L\to \mathbb{R} $ is \emph{monotone} if $ \forall x,y\in L $ such that $ x\leq_L y $, we have $ f(x)\leq f(y) $, 
where $ \leq_L $ is the partial order defined by  lattice $ L $\footnote{In a 
lattice, we define $ a\leq_{L} b  $ if $ a=a\wedge b 
$~\citep{sankappanavar1981course}}.

For any set $ E $, its power set  $ 2^E $ equipped with set union $ \cup $ and 
intersection $ \cap $ is an instance of lattice. In fact, submodular functions 
on the lattice $ (2^E,\cup, \cap) $ are precisely the submodular set functions 
that have been extensively studied in the past 
~\citep{nemhauser1978analysis,fujishige2005submodular}.
If we let $ [C] $ denote $ \{1,2,3,\ldots,C \} $, then $ [C]^n $ and $ 
\mathbb{Z}^n $ are bounded and unbounded integer lattices equipped with 
entrywise maximum ($ \vee $) and minimum ($ \wedge $). This construction 
corresponds to submodular functions on integer 
lattices~\citep{gottschalk2015submodular,soma2016maximizing}.

\paragraph{Continuous Submodularity}
In contrast to the above discrete scenarios, we focus on continuous domains in this paper.
The set $\mathcal{X}\triangleq\prod_{i=1}^{n}\mathcal{X}_{i}\subseteq\mathbb{R}_+^{n}$,
where $\mathcal{X}_{i}$'s are closed intervals of $\mathbb{R}_+$, is also equipped
with a natural lattice structure where $\vee$ and $\wedge$ are entrywise
maximum and entrywise minimum, respectively, i.e., for any ${\bf x},{\bf y}\in\mathcal{X}\subseteq\mathbb{R}^{n}$,
the $i$-th component of ${\bf x}\vee{\bf y}$ is $\max\{x_{i},y_{i}\}$
and the $i$-th component of ${\bf x}\wedge{\bf y}$ is $\min\{x_{i},y_{i}\}$.
A function $f:\prod_{i=1}^n \mathcal{X}_i \to \mathbb{R}_+$ is called \emph{continuous submodular} if it is submodular under this lattice.
When the function $f$ is
twice differentiable, it is continuous submodular if and only if all
off-diagonal entries of its Hessian are non-positive,
i.e., 
\[
\forall i\neq j,\forall{\bf x}\in\mathcal{X},\frac{\partial^{2}f({\bf x})}{\partial x_{i}\partial x_{j}}\leq0.
\]
Without loss of generality, we assume that $\mathcal{X}_i=[0,b_i]$, $\forall 
1\leq i\leq n$. If $\mathcal{X}_i=[c_i,d_i]$ and $f$ is continuous submodular 
on $\prod_{i=1}^n [c_i,d_i]$, we can consider another continuous submodular 
function $\tilde{f}$ defined on $\prod_{i=1}^n [0,d_i-c_i]$ such that 
$\tilde{f}(\x)=f(\x+\mathbf{c})$.

\paragraph{DR-Submodularity}
In this paper, we are mainly interested in 
a subclass of differentiable continuous submodular functions that 
exhibit diminishing returns~\citep{bian16guaranteed}, i.e., for every $ 
\mathbf{x},\mathbf{y}\in \mathcal{X} $, $ 
\mathbf{x}\leq \mathbf{y} $ elementwise implies 
$$ \nabla f({\bf x})\geq\nabla f({\bf y})
$$
elementwise, which indicates that the gradient is an antitone 
mapping~\citep{bian16guaranteed,eghbali2016designing}.  When the 
 function $f$ is twice differentiable,  DR-submodularity is equivalent to 
 \[
\forall i,j,\forall{\bf x}\in\mathcal{X},\frac{\partial^{2}f({\bf x})}{\partial x_{i}\partial x_{j}}\leq0.
\]
Twice differentiable DR-submodular functions are also called \textit{smooth} 
submodular functions~\citep{vondrak2007submodularity}.
 
%
%
We say that a function $f$ is \textit{weakly} DR-submodular with parameter 
$\gamma$ \citep{hassani2017gradient}
if 
\[
\gamma=\inf_{{\bf x},{\bf y}\in\mathcal{X},{\bf x}\leq{\bf y}}\inf_{i\in[n]}\frac{[\nabla f({\bf x})]_{i}}{[\nabla f({\bf y})]_{i}},
\]
where $ [\nabla f(\mathbf{x})]_i  = \frac{\partial 
f(\mathbf{x})}{\partial x_i}$ is the $ i $-th component of the 
gradient. If the function is monotone, we have $ \gamma \geq 0 $.
Note that a differentiable DR-submodular function is weakly submodular
with parameter $\gamma=1$. 

In this work, we focus on monotone continuous (weakly) DR-submodular 
functions.

\paragraph{Multilinear Extension}
An important example of continuous DR-submodular functions is the multilinear extension of a submodular set function. Given a monotone submodular set function $ W :2^{\Omega} \to \mathbb{R}_+ $ defined on a ground set $\Omega$, its multilinear extension $ \bar{f}:[0,1]^{|\Omega|} \to \mathbb{R} $ is defined as 
 \[ \bar{f}(\mathbf{x}) = \sum_{S \subseteq \Omega } W(S) \prod_{i\in S} x_i 
 \prod_{ j \notin S } (1-x_j),  \] is monotone 
 DR-submodular~\citep{calinescu2011maximizing}. In general, it is 
 computationally intractable to compute the multilinear extensions. However, 
 for the weighted coverage functions~\citep{karimi2017stochastic}, they have an 
 interesting connection to concavity. 
Suppose that $ U $ is a finite set and let $ G:2^U\to \mathbb{R} $ be a nonnegative modular function such that $ G(S)\triangleq \sum_{u\in U} w(u) $, where $ w(u)\geq 0 $ for all $ u\in U $. We have a finite collection $ \Omega = \{B_i: 1\leq i\leq n \} $ of subsets of $ U $. The weighted coverage function $ W :2^{\Omega} \to \mathbb{R}_{\geq 0} $ is defined as \[ 
W(S) \triangleq G( \bigcup_{B_i \in S} B_i ),\forall S\subseteq \Omega.
\]
\citet{karimi2017stochastic} showed that the multilinear extension $ f:[0,1]^{n} \to \mathbb{R} $ is 
\[ 
\bar{f}(\mathbf{x}) = \sum_{u\in U} w(u) \left( 1-\prod_{B_i\in \Omega:u\in B_i } (1-x_i) \right).
\]
They showed that the multilinear extension has a concave upper bound. In fact, in light of the Fenchel concave biconjugate, they consider a concave function \[
\tilde{f}(\mathbf{x}) \triangleq \sum_{u\in U} w(u) \min\left\{1, \sum_{ B_i\in \Omega : u\in B_i } x_i \right\}
\]   
and showed a key squeeze relation
\[ (1-1/e)\tilde{f}(\mathbf{x}) \leq \bar{f}(\mathbf{x}) \leq 
\tilde{f}(\mathbf{x}), \quad \forall \mathbf{x}\in [0,1]^n. \]

\subsection{Online Continuous Submodular Maximization}\label{sub:protocol}

\begin{figure}[bth]
	\begin{framed}
		\begin{algorithmic}[1]
			\Require{convex set $\mathcal{P}$, horizon $T$} 
			\Ensure{$\{\mathbf{x}_t:1\leq t\leq T \}$}
			\State Determine $\mathbf{x}_1\in \mathcal{P}$ \Comment{\emph{to be designed}}
			\For{$t\gets 1,2,3,\ldots, T$}
			\State{Play $\mathbf{x}_t$, observe reward $f_t(\mathbf{x}_t)$}
			\State Observe $f_t$  and determine $\mathbf{x}_{t+1}\in \mathcal{P}$ \Comment{\emph{to be designed}}
			\EndFor
		\end{algorithmic}
	\end{framed}
\end{figure}


	The general protocol of online continuous submodular maximization
is given as follows. At iteration $t$ (going from $1$ to $T$), the 
online algorithm chooses $\mathbf{x}_t\in \mathcal{P}$. After 
committing to this choice, a monotone DR-submodular function $f_t$ is 
revealed and the algorithm receives the reward $f_t(\mathbf{x}_t)$. 
The goal is to minimize \textit{regret} which is typically defined as 
the difference between the total award that the algorithm accumulated and 
that of the best fixed decision in hindsight. Note that even in the 
offline setting, maximizing a monotone DR-submodular function subject 
to a convex  constraint can only be done approximately in polynomial 
time unless $ \mathbf{RP}=\mathbf{NP} $~\citep{bian16guaranteed}. 
Thus, we instead 
define 
the \emph{$ \alpha $-regret} of an algorithm $ \mathcal{A} $ as 
follows~\citep{streeter2009online, kakade2009playing}: 
\[ 
\mathcal{R}_{\alpha}(\mathcal{A}, T) \triangleq \alpha \max_{\x \in \mathcal{P}} \sum_{t=1}^{T} f_t(\mathbf{x}) - \sum_{t=1}^{T} f_t(\x_t),
 \]
 where $ \alpha $ is  the approximation ratio. In the deterministic setting when full access to the gradients of $f_t$'s is possible, the best polynomial-time approximation guarantee in the offline setting is $\alpha=1-1/e$,
  using a variant of the Frank-Wolfe algorithm,
  unless $ \mathbf{RP}=\mathbf{NP} $~\citep{bian16guaranteed}. In contrast, for 
  the stochastic situations where only unbiased estimates of gradients are 
  given, the best known approximation guarantee (in the offline setting) is 
  $\alpha=1/2$ \citep{hassani2017gradient}, using stochastic gradient ascent. 
  It is also known that stochastic gradient ascent cannot achieve a better 
  approximation guarantee in general~\citep{hassani2017gradient, 
  vondrak2011submodular}. 
\section{Algorithms and Main Results}
In this section, we describe our online algorithms \OFW  and \OGA for a sequence of monotone DR-submodular functions,  in the no-regret setting. 

\subsection{$(1-1/e)$ Guarantee via \OFW}
We begin by proposing the \OFW algorithm that achieves $ (1-1/e) $ fraction of the global maximum in hindsight up to $ O(\sqrt{T}) $ regret.
Our algorithm is based on the Frank-Wolfe variant proposed 
in~\citep{bian16guaranteed} for maximizing monotone and continuous 
DR-submodular functions and the idea of meta-actions proposed 
in~\citep{streeter2009online}. Unlike \citep{bian16guaranteed}, we 
consider a general convex body $\mathcal{P}$ as the constraint set 
and do not assume that it is down-closed.  We use meta-actions to 
convert offline algorithms into online algorithms. To be precise, 
let us consider the first iteration and the first objective function 
$f_1$ of our online optimization setting. Note that $f_1$ remains 
unknown until the algorithm commits to a choice. If we were in the 
offline setting, we could have used the Frank-Wolfe variant proposed 
in~\citep{bian16guaranteed}, say ran it for $k$ iterations, in order 
to maximize $f_1$. In each iteration, we would have found a vector 
$\mathbf{v}_k\in \mathcal{P}$ that maximizes $\langle \mathbf{v}_k, 
\nabla f_1(\mathbf{x}_k)\rangle$ and performed the update 
$$\mathbf{x}_{k+1}\gets \mathbf{x}_k + \frac{1}{K} \mathbf{v}_k.$$ The 
idea of meta-actions is to mimic this process in an online setting as 
follows. We run $K$ instances $\{\mathcal{E}^k: 1\leq k\leq K\}$ of 
an off-the-shelf online linear maximization algorithm, such as  
Regularized-Follow-The-Leader (RFTL)~\citep{hazan2016introduction}. 
Here $K$ denotes the number of iterations of the offline Frank-Wolfe 
algorithm that we intend to mimic. Thus, to maximize $\langle \cdot, 
\nabla f_1(\mathbf{x}_k)\rangle$, where $\nabla f_1(\mathbf{x}_k)$ is 
the unknown linear objective function of the online linear 
maximization problem, we simply use $\mathcal{E}^k$. Once the 
function $f_1$ is revealed to the algorithm, it knows each linear 
objective function $\nabla f_1(\mathbf{x}_k)$ and its corresponding 
inner product $\langle \mathbf{v}_k, \nabla f_1(\mathbf{x}_k)\rangle$. 
Now, we simply feed each online algorithm $ \mathcal{E}_k $ with the 
reward $ \langle \mathbf{v}_k, \nabla f_1(\mathbf{x}_k)\rangle $. 
 For any subsequent function $f_t$ ($t\geq 2$), we repeat the above 
 process. Note that for an RFTL algorithm the regret is bounded by 
 $O(\sqrt{T})$ (in fact, this is true for many choices of no-regret 
 algorithms). This idea combined with the fact that the Frank-Wolfe 
 algorithm can be used to maximize a monotone and continuous 
 DR-submodular function and attain $(1-1/e)$ fraction of the optimum 
 solution suffices to prove that $(1-1/e)$-regret of \OFW is also 
 bounded by $O(\sqrt{T})$. The precise description of  \OFW is 
 outlined in~\cref{alg:Online-Continuous-Submodular}. Recall that the 
 \emph{positive 
 orthant} of the Euclidean space $ \mathbb{R}^n $ is $ \{ \mathbf{x}\in 
 \mathbb{R}^n: x_i\geq 0, \forall 1\leq i\leq n \} $.

\begin{algorithm}[bth]
	\begin{algorithmic}[1] 
		\Require{
		 $\mathcal{P}$ is a 
        convex set in the positive orthant, and  
        $T$ is the horizon.
        } 
		\Ensure{$\{\mathbf{x}_t:1\leq t\leq T \}$} 
		\State{Initialize $K$ Regularized-Follow-The-Leader (RFTL) algorithm instances $\{\mathcal{E}^{k}: 0\leq k < K \}$ for maximizing linear cost functions over $\mathcal{P}$ }
		\For{$t\gets 1,2,3,\ldots, T$} 
		\For{$k \gets 0,1,2,\ldots, K-1$}
		\State{Let $\mathbf{v}^{k}_t $ be the vector selected by $\mathcal{E}^{k}$}
		\EndFor
		\State{$\mathbf{x}_t \gets \frac{1}{K} \sum_{k=0}^{K-1}  
		\mathbf{v}^k_t$}
		\State{Play $ \mathbf{x}_t $, receive reward $f_t(\mathbf{x}_t)$ and 
		observe $f_t$} 
		\State{$\forall 0\leq k \leq K, \mathbf{x}_t(k)\gets 
		1_{\{k>0\}} \frac{1}{K} \sum_{s=0}^{k-1}  \mathbf{v}^s_t $}
		\For{$k\gets 0,1,2,\ldots, K-1$}
		\State{Feed back $\langle \mathbf{v}^{k}_t , \nabla f_t(\mathbf{x}_t(k))\rangle$ as the payoff to be received by $\mathcal{E}^k$}
		\EndFor
		\EndFor
	\end{algorithmic}\caption{\textsf{Meta-Frank-Wolfe} 
	\label{alg:Online-Continuous-Submodular}}
\end{algorithm}

In the following theorem, we bound the $(1-1/e)$-regret of \OFW.

\begin{thm}\label{thm:guarantee-frank-wolfe}
	\emph{\textbf{(Proof in~\cref{sec:proof-frank-wolfe})}}
	Assume that $ f_t $ is monotone DR-submodular and $ \beta $-smooth for 
	every $ t $. 
	By using Algorithm~\ref{alg:Online-Continuous-Submodular},
	 we obtain 
	\begin{align*}
	& (1-1/e)\sum_{t=1}^{T}f_{t}({\bf x}^{*})-\sum_{t=1}^{T}f_{t}({\bf x}_{t})\\
	\leq & -e^{-1}\sum_{t=1}^{T}f_{t}(0)+2DG\sqrt{T}+\frac{\beta R^2 T}{2K},
	\end{align*}
	where 
	$ D=\diam(\mathcal{P}) $, $ R=\rho(\mathcal{P}) $,
	  and $G=\sup_{1\le t\leq T,{\bf x}\in\mathcal{P}}\lVert\nabla 
	  f_{t}({\bf x})\rVert$ are assumed to be finite.
\end{thm}


If we assume that the functions $f_t$ are non-negative, then we have $f_t(0)\geq 0$ for all $t$, which implies that 
the first term $-e^{-1}\sum_{t=1}^{T}f_{t}(0)$ in the regret bound of \cref{thm:guarantee-frank-wolfe} is non-positive (thus reduces the entire sum). The second term is $O(\sqrt{T})$. Finally, If we let the number of RFTL algorithm instances $K$ be equal to $\sqrt{T}$, the final term $\frac{\beta R^2 T}{2K}$ will become $\frac{\beta R^2 }{2}\sqrt{T}$.  


\subsection{$1/2$ Guarantee via \OGA}\label{sub:grad-ascent}


We saw that when the gradient can be efficiently evaluated, \OFW presented in
\cref{alg:Online-Continuous-Submodular}
yields a sublinear regret bound. However, efficient evaluation of the gradient 
could be impossible in many scenarios. 
For example, exact evaluation of the gradients of the multilinear extension  of a submodular set function  requires summation over exponentially many terms. 
Furthermore, one may consider a class of stochastic continuous 
DR-submodular functions $f(\x)=\mathbb{E}_{\theta\sim 
\mathcal{D}}[f_{\theta}(\x)]$, where every $f_{\theta}$ is continuous 
DR-submodular and the parameter $\theta$ is sampled from a 
(potentially unknown) distribution $\mathcal{D}$ 
\citep{hassani2017gradient,karimi2017stochastic}. Again, in such cases 
it is generally intractable to compute the gradient of $f(\x)$, 
namely, $\nabla f(\x)=\mathbb{E}_{\theta\sim \mathcal{D}}[\nabla 
f_{\theta}(\x)]$\footnote{This equation holds if some regularity 
conditions are satisfied, in light of Lebesgue's dominated 
convergence theorem.}. Instead, the stochastic terms $\nabla 
f_{\theta}(\x)$ provide unbiased estimates for the gradients. 
Another disadvantage of the \OFW algorithm is that it requires $ 
O(\sqrt{T}) $ gradient queries for each function $ f_t $, which may 
be even more prohibitive.
In this 
subsection, we show how we can use \OGA to design an algorithm with 
sublinear regret and robust to stochastic gradients when the 
functions $f_t$ are monotone and 
continuous DR-submodular. 


First, it was shown by \citet{hassani2017gradient} that a direct usage of  
unbiased estimates of the gradients in  Frank-Wolfe-type algorithms can lead to 
arbitrarily bad solutions in the context of stochastic submodular 
maximization. This happens due to the non-vanishing variance of gradient 
approximations. As a  result, new techniques should be developed for the online 
optimization algorithm with access to  unbiased estimates of the gradients of 
$f_t$ (instead of the exact gradients). 
To handle the stochastic noise in the gradient, we consider the (stochastic) gradient ascent method. In \cref{thm:gradient-ascent}, we show  that  
the $(\frac{\gamma^2}{\gamma^2+1})$-regret of 
(stochastic) \OGA
is bounded by $O(\sqrt{T})$
for $ \gamma $-weakly DR-submodular functions. 
In particular, for the special case of $ \gamma=1 $, 
the $1/2$-regret of \OGA is bounded by $O(\sqrt{T})$
for continuous DR-submodular functions. 
The precise description of \OGA is presented in \cref{alg:Online-Gradient} while its stochastic version is presented in \cref{alg:Online-Stoch-Gradient}.

\begin{algorithm}[htb]
\begin{algorithmic}[1] 
\Require{convex set $\mathcal{P}$, $T$, $\mathbf{x}_1\in \mathcal{P}$, step sizes $\{ \eta_t \}$} 
\Ensure{$\{\mathbf{x}_t:1\leq t\leq T \}$} 
\For{$t\gets 1,2,3,\ldots, T$}
\State{Play $\mathbf{x}_t$ and receive reward $f_t(\mathbf{x}_t)$.}
\State{$\mathbf{x}_{t+1}=\Pi_{\mathcal{P}}(\mathbf{x}_t+\eta_t \nabla f_t(\mathbf{x}_t) )$}
\EndFor
\end{algorithmic}\caption{\textsf{Online Gradient 
Ascent}\label{alg:Online-Gradient}}
\end{algorithm}
\begin{algorithm}[htb]
	\begin{algorithmic}[1] 
		\Require{convex set $\mathcal{P}$, $T$, $\mathbf{x}_1\in \mathcal{P}$, step sizes $\{ \eta_t \}$} 
		\Ensure{$\{\mathbf{x}_t:1\leq t\leq T \}$} 
		\For{$t\gets 1,2,3,\ldots, T$}
		\State{Play $\mathbf{x}_t$ and receive reward $f_t(\mathbf{x}_t)$.}
		\State{Observe $\mathbf{g}_t$ such that $\mathbb{E}[\mathbf{g}_t|\mathbf{x}_t]=\nabla f_t(\mathbf{x}_t)$}
		\State{$\mathbf{x}_{t+1}=\Pi_{\mathcal{P}}(\mathbf{x}_t+\eta_t \mathbf{g}_t )$}
		\EndFor
	\end{algorithmic}\caption{\textsf{Online Stochastic Gradient 
	Ascent}\label{alg:Online-Stoch-Gradient}}
\end{algorithm}

\begin{thm}\label{thm:gradient-ascent}
\textbf{\emph{(Proof in~\cref{sec:proof-gradient}})}
Assume that the functions $f_t:\mathcal{X}\to\mathbb{R}_{+}$ are monotone and weakly DR-submodular
 with parameter $\gamma$ for $ t= 1,2,3,\dots, T $. Let $ \{ 
 \mathbf{x}_t: 1\leq t\leq T \} $ be the choices of 
 \cref{alg:Online-Gradient} (\cref{alg:Online-Stoch-Gradient}, 
 respectively) and let $ \eta_t = 
 \frac{D}{G\sqrt{t}} $, then we have
\[
\frac{\gamma^{2}}{\gamma^{2}+1}\sum_{t=1}^{T}f_{t}({\bf x}^{*})-\sum_{t=1}^{T}f_{t}({\bf x}_{t})\leq\frac{3\gamma DG\sqrt{T}}{2(\gamma^{2}+1)}
\]
and
\[
\frac{\gamma^{2}}{\gamma^{2}+1}\sum_{t=1}^{T}f_{t}({\bf x}^{*})-\sum_{t=1}^{T}\mathbb{E}\left[f_{t}({\bf x}_{t})\right]\leq\frac{3\gamma DG\sqrt{T}}{2(\gamma^{2}+1)}.
\]
for \cref{alg:Online-Gradient} and \cref{alg:Online-Stoch-Gradient}, respectively, where 
$ D=\diam(\mathcal{P}) $
 and $ G = \sup_{ 1\leq t\leq T, \mathbf{x}\in \mathcal{P}} \lVert 
 \nabla f_t(\mathbf{x}) \rVert $ (for 
 \cref{alg:Online-Stoch-Gradient}, $ G = \sup_{ 1\leq t\leq T
 } \lVert 
  \mathbf{g}_t \rVert $) are assumed to be finite. In 
 particular, when $ f_t $ is continuous DR-submodular ($ \gamma=1 $), 
 we have \[ \frac{1}{2}\sum_{t=1}^{T}f_{t}({\bf 
 x}^{*})-\sum_{t=1}^{T}f_{t}({\bf x}_{t})\leq\frac{3}{4} DG\sqrt{T} \]
and
\[ \frac{1}{2}\sum_{t=1}^{T}f_{t}({\bf x}^{*})-\sum_{t=1}^{T} \mathbb{E}[f_{t}({\bf x}_{t})]\leq\frac{3}{4} DG\sqrt{T}, \]
respectively.
\end{thm}



\section{Experiments}\label{sec:experiments}


In the experiments, we compare the performance of the following algorithms:
\begin{itemize}
\item \OFW. 
We choose $ r(\mathbf{x}) = \lVert \mathbf{x}-\mathbf{x}_0\rVert ^2/2 $ as the 
regularizer of the RFTL in \OFW. RFTL has a parameter $\eta$ that balances the 
sum of inner products with the gradients of each step and the 
regularizer~\citep{hazan2016introduction}.
\item \OGA. We also denote the step size (also known as the learning rate) of the online gradient ascent by $ \eta $. Therefore \OGA also has a parameter $ \eta $.
\item \Rand. For each objective function $ f_t $, \Rand samples 100 points in 
the constraint set and selects the one that maximizes $ f_t $. We would like to 
emphasize that \Rand is \emph{infeasible} in the online setting since online 
algorithms have to make decisions before an objective function is revealed.
\item \SGA. When the objective functions are the multilinear extension of 
submodular \emph{coverage} functions, we also studied the performance of 
gradient ascent applied to a surrogate function, which is shown to be a concave 
upper bound for the multilinear extension~\citep{karimi2017stochastic}. 
\end{itemize}
\begin{figure*}[t]
	\centering
	\begin{subfigure}{0.32\textwidth}
		\includegraphics[width=\textwidth]{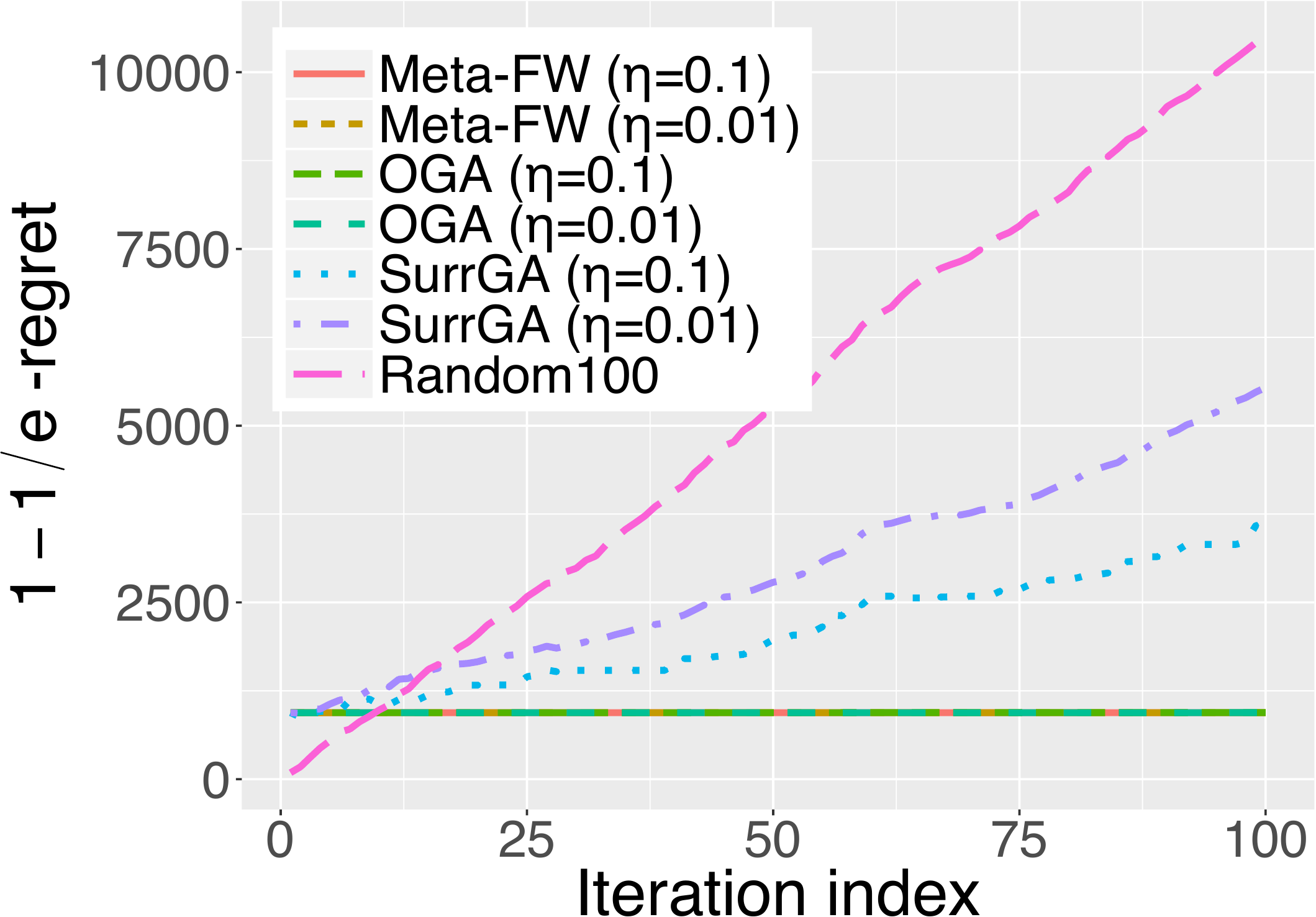}
		\caption{\label{fig:coverage_regret_T}}
	\end{subfigure}
	\begin{subfigure}{0.32\textwidth}
		\includegraphics[width=\textwidth]{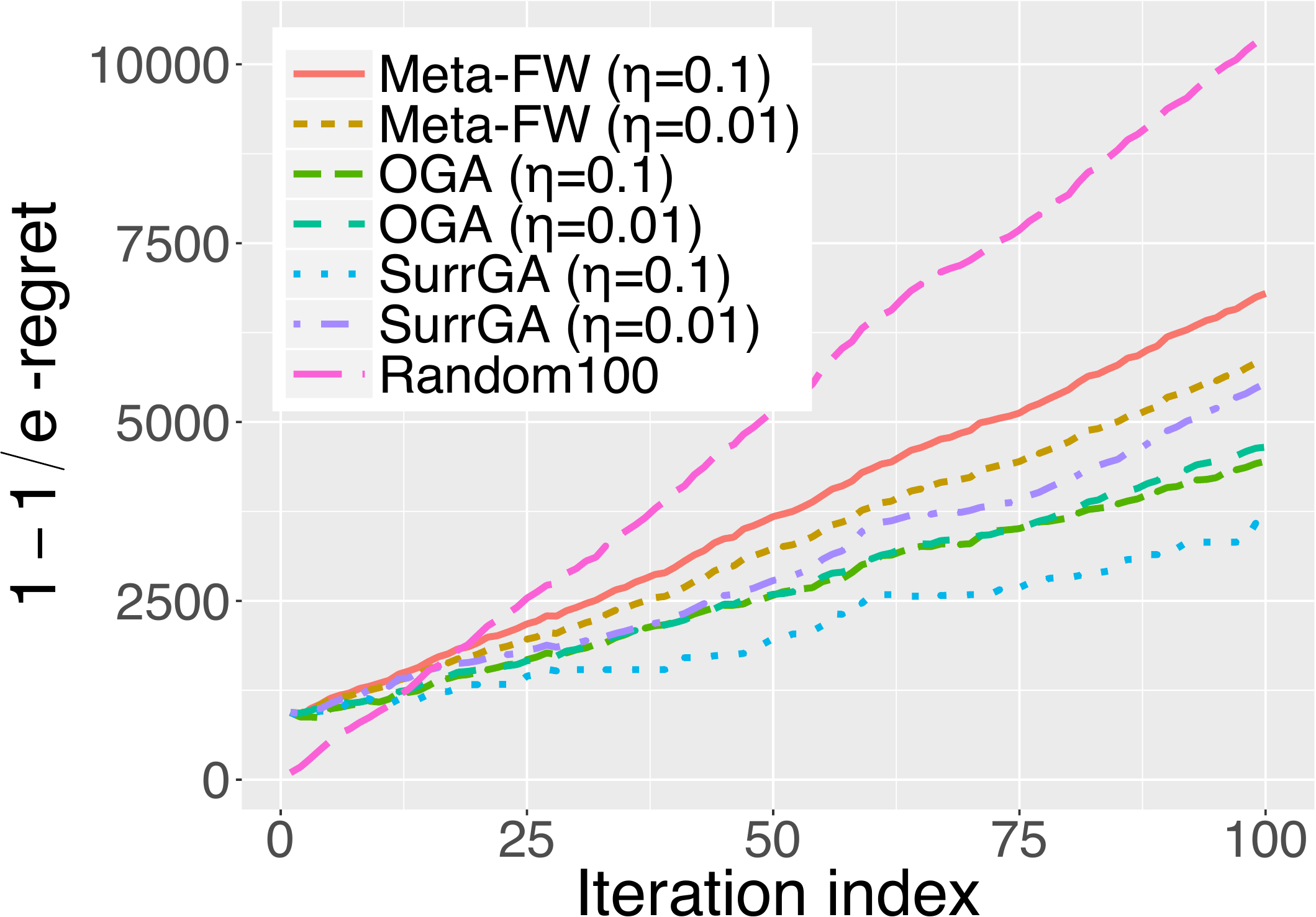}
		\caption{\label{fig:coverage_regret_stoch_T}}
	\end{subfigure}
	\begin{subfigure}{0.32\textwidth}
		\includegraphics[width=\textwidth]{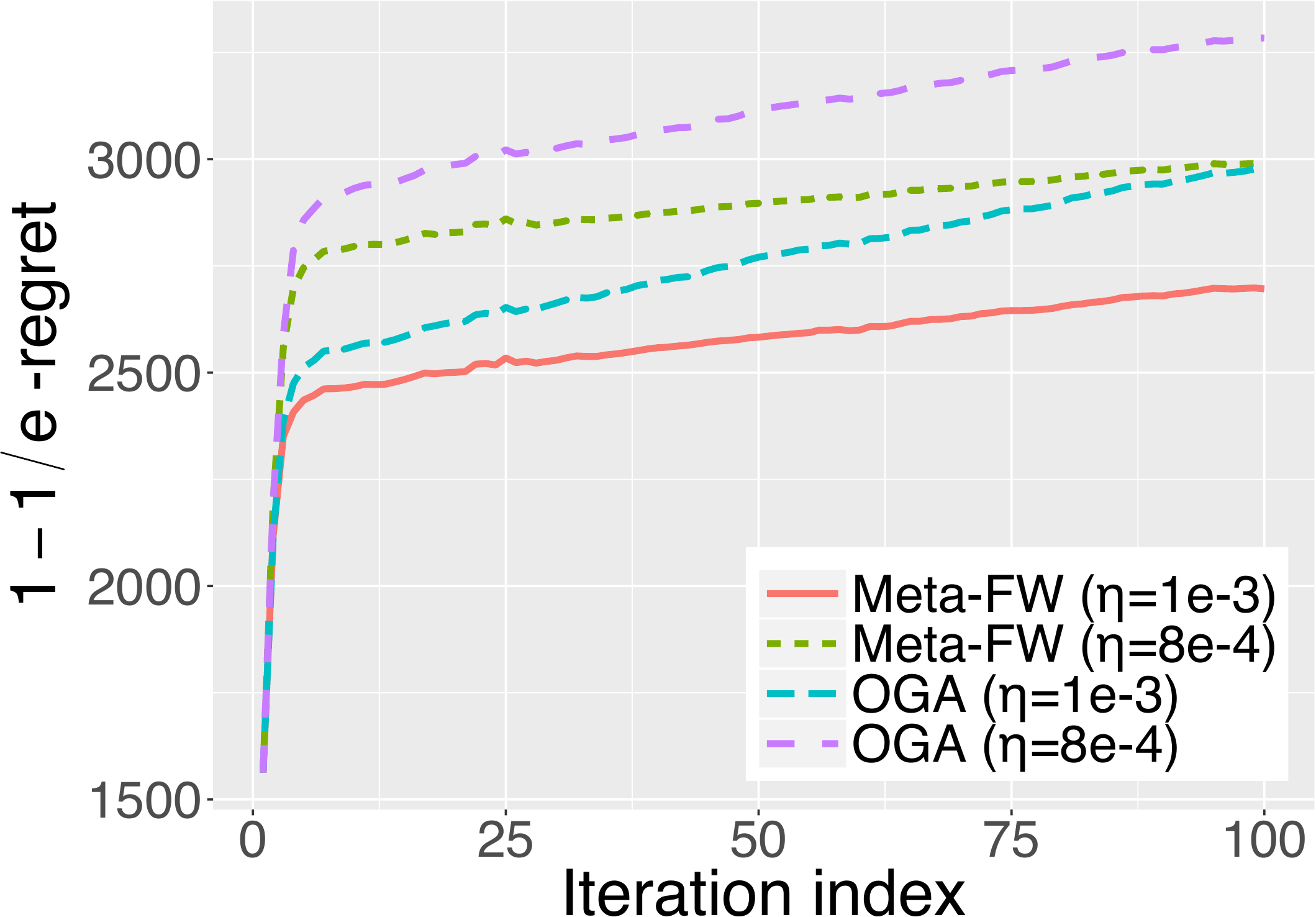}
		\caption{\label{fig:nqp_T}}
	\end{subfigure}
%
	
	\begin{subfigure}{0.32\textwidth}
		\includegraphics[width=\textwidth]{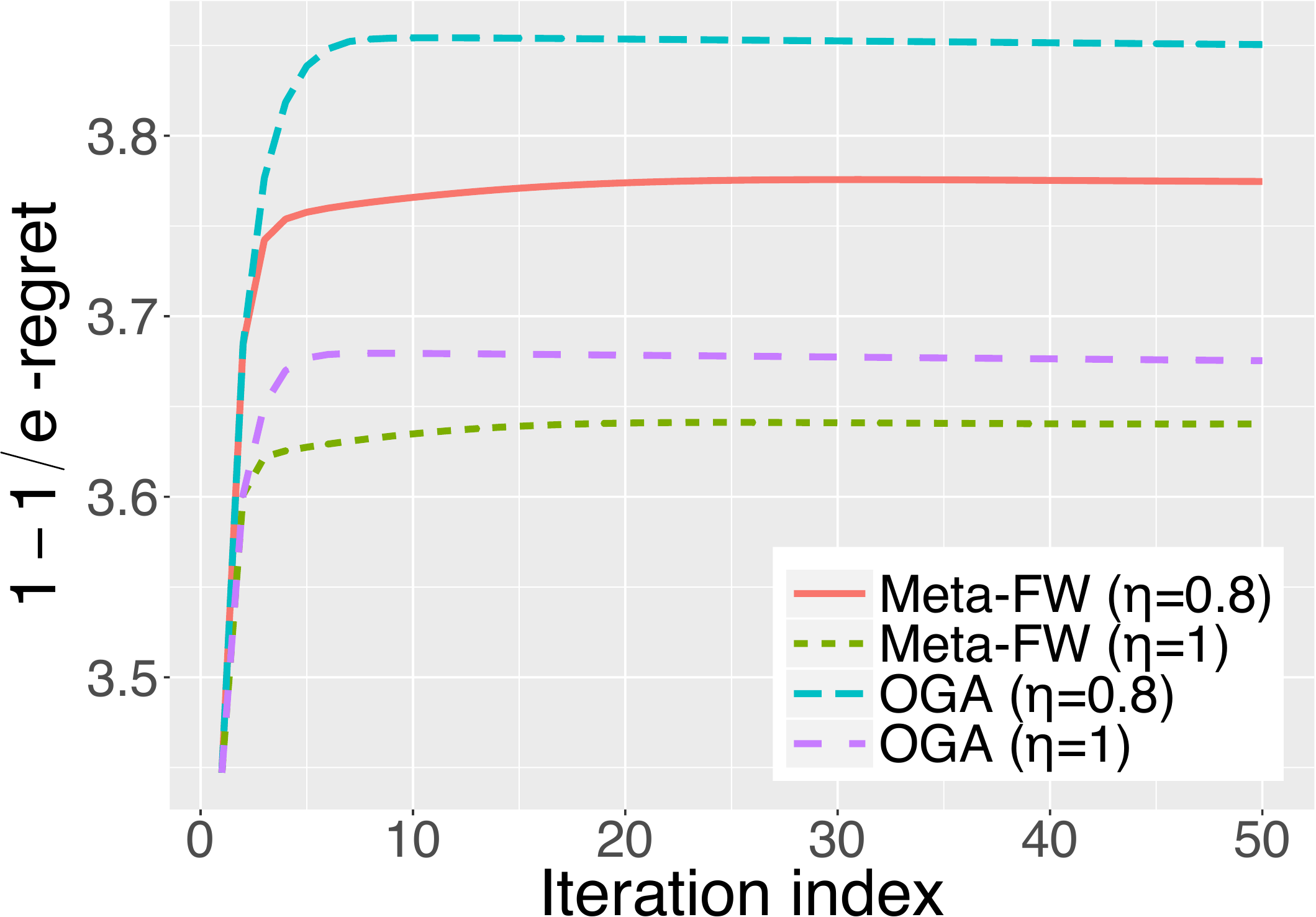}
		\caption{\label{fig:d-optimal_T}}
	\end{subfigure}
	\begin{subfigure}{0.32\textwidth}
		\includegraphics[width=\textwidth]{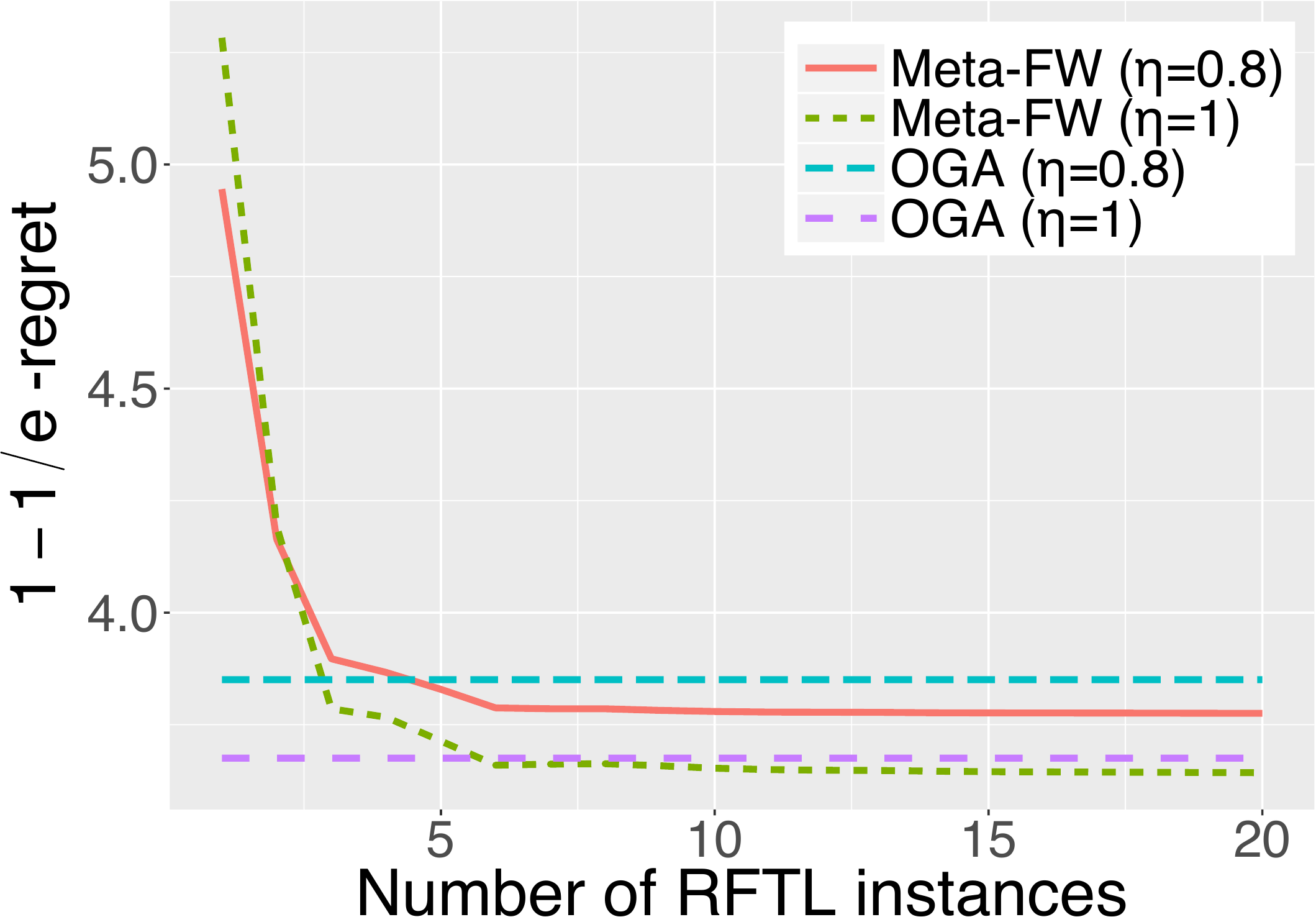}
		\caption{\label{fig:d-optimal_K}}
	\end{subfigure}
	\caption{In the legends of all subfigures, we write \textsf{Meta-FW} for 
	\OFW, \textsf{OGA} for \OGA, and \textsf{SurrGA} for \SGA.
	The results for the multilinear extension are 
	presented in \cref{fig:coverage_regret_T,fig:coverage_regret_stoch_T}. We 
	present the $ (1-1/e) $-regret versus the number of iterations in 
	\cref{fig:coverage_regret_T}. In \cref{fig:coverage_regret_stoch_T} we 
	illustrate  the result for the setting in which only an unbiased estimate 
	of the gradient is available.
		\cref{fig:nqp_T} shows how the $ (1-1/e) $-regret evolves for the 
		non-convex/non-concave quadratic programming.
		\cref{fig:d-optimal_T,fig:d-optimal_K} are about the D-optimal experiment design problem. \cref{fig:d-optimal_T} shows the $ (1-1/e) $-regret versus the number of iterations, while \cref{fig:d-optimal_K} shows how the number of RFTL instances $ K $ influences the performance of \OFW.}
\end{figure*}

\subsection{Multilinear Extension}
As our first experiment, we consider a sequence of multilinear extensions of 
weighted coverage functions (see \cref{sec:submodularity}). Recall that such 
functions have a concave lower bound. Thus, we introduce another baseline \SGA 
that uses supergradient ascent to maximize the concave lower bound  function $ 
(1-1/e)\bar{f}(\mathbf{x}) $.
The result is presented in \cref{fig:coverage_regret_T}. We observe that \Rand has the highest regret and both \OFW and \OGA, whose performance is slightly inferior to that of \OFW, outperform \SGA.

Then, we study the case where only an unbiased estimate of the gradient is available. For any $ \mathbf{x}\in [0,1]^n $, let $$ [ \tilde{ \nabla} f(\mathbf{x})]_i \triangleq f(R_i \cup\{i\} ) -f(R_i) ,$$
where $ R_i $ is a random subset of $ [n] \setminus\{i\} $ such that each $ 
j\neq i $ is in $ R_i $ with probability $ x_j $ independently. Then  we have $ 
\mathbb{E}[ \tilde{\nabla} f(\mathbf{x})] = \nabla f(\mathbf{x}) 
$~\citep{calinescu2011maximizing}.
The result in this setting is presented in~\cref{fig:coverage_regret_stoch_T}. 
Notice that in \cref{fig:coverage_regret_stoch_T} the regret of \Rand and \SGA is uninfluenced by the stochastic gradient oracle since they do not rely on the exact gradient of the original objective function. \OFW and \OGA both incur higher regret in \cref{fig:coverage_regret_stoch_T} than in \cref{fig:coverage_regret_T}. In addition, the stochastic gradient oracle has more impact upon \OFW than \OGA. This agrees with our theoretical guarantee for \OGA and a result from
\citep{hassani2017gradient}, which states that Frank-Wolfe-type algorithms are 
not robust to stochastic noise in the gradient oracle.

\subsection{Non-Convex/Non-Concave Quadratic Programming}

Quadratic programming problems have objective functions of the form $ f(\mathbf{x}) = \frac{1}{2}\x^{\top} \mathbf{H} \x + \mathbf{h}^{\top} \x + c $ and linear equality and/or inequality constraints. If the matrix $ \mathbf{H} $ is indefinite, the objective function becomes non-convex and non-concave. We constructed $ m $ linear inequality constraints $ \mathbf{A} \x \leq \mathbf{b} $, where each entry of $ \mathbf{A}\in \mathbb{R}^{m\times n} $ is sampled uniformly at random from $ [0,1] $. We set $ m=2 $. In addition, we require that the variable $ \x $ reside in a positive cuboid. Formally, the constraint is a positive polytope $ \mathcal{P}=\{ \x\in \mathbb{R}^n: \mathbf{A}\x \leq \mathbf{b}, 0\leq \x \leq \mathbf{u} \} $. We set $ \mathbf{b} = \mathbf{u} = \mathbf{1} $. To ensure that the gradient is non-negative, we set $ \mathbf{h} = -\mathbf{H}^{\top}\mathbf{u} $. Without loss of generality, we assume that the constant term $ c $ is $ 0 $. Thus the function is $ f(\mathbf{x}; \mathbf{H}) = (\frac{1}{2}\mathbf{x}-\mathbf{u})^{\top} \mathbf{H}\mathbf{x} $; it is fully determined by the matrix $ \mathbf{H} $. In our online optimization setting, we assume that the $ T $ functions $ f_1, f_2, \ldots, f_T $ are associated with matrices $ \mathbf{H}_1, \mathbf{H}_2, \ldots, \mathbf{H}_T $. For every $ \mathbf{H}_i $, its entries are sampled uniformly at random from $ [-100,0] $.
 We set $ K=50 $. The result is illustrated in
\cref{fig:nqp_T}. It can be observed that with the same step size $ \eta $, the 
regret of \OFW is smaller than \OGA.

\subsection{D-Optimal Experimental Design}

The objective function of the D-optimal design problem is 
	$
	f(\bm{\lambda})
	=\log\det\left(\sum_{i=1}^{N}\lambda_{i}{\bf x}_{i}{\bf 
	x}_{i}^{\top}\right).
$ 
 We
write $A(\bm{\lambda})$ for $\sum_{i=1}^{N}\lambda_{i}{\bf x}_{i}{\bf x}_{i}^{\top}$
for the ease of notation.
It is DR-submodular because for any $ i $ and $ j $
\[
\frac{\partial^{2}f(\bm{\lambda})}{\partial\lambda_{j}\partial\lambda_{i}}=-({\bf x}_{j}^{\top}A(\bm{\lambda})^{-1}{\bf x}_{i})^{2}\leq0.
\]
For every $ \mathbf{x}_i $, its entries are sampled from the standard normal distribution independently.
We try to solve the maximization in the polytope $ \mathcal{P} = \{ \bm{\lambda}: \mathbf{A}(\bm{\lambda}-\bm{1}) \leq \bm{1}, \bm{1}\leq \bm{\lambda}\leq \bm{2}  \} $. Each entry of $ \mathbf{A} $ is sampled uniformly from $ [0,1] $ and the number of inequality constraints is set to $ 2 $. The polytope is shifted to avoid $ \bm{0} $ since the function is undefined at $ \bm{\lambda}=\bm{0} $.
 In \cref{fig:d-optimal_T}, we illustrate how the function value attained by the algorithms varies as it experiences more iterations; $ K $ is fixed to be $ 50 $ in this set of experiments. 
 We observe that \OFW outperforms all other baselines. In addition, \OFW 
 achieves better performance when the step size $ \eta=1 $.

 In the second set of experiments, we show the function values attained by the 
 algorithms at the end of the $ 50 $th iteration, with $ K $ ranging from $ 1 $ 
 to $ 20 $ for \OFW. 
 Recall that $ K $ is the number of Frank-Wolfe steps in \OFW.
 The result is presented in \cref{fig:d-optimal_K}. 
 Since  $ K $ is not a parameter of  \OGA, the regret of \OGA remains constant as $ K $ varies. The regret of \OFW is reduced as $ K $ increases.  This agrees with our intuition that more Frank-Wolfe steps yield better performance.

%
%
%
%

%
%
%
%
\section{Related Work}
\paragraph{Submodular functions.} Submodularity is a structural property that 
is often associated with set functions \citep{nemhauser1978analysis, 
fujishige2005submodular}. It has found far-reaching applications in statistics 
and artificial intelligence, including active learning 
\citep{golovin2011adaptive}, viral marketing 
\citep{kempe2003maximizing,gomez2012influence,Zhang2016Influence}, network 
monitoring 
\citep{leskovec2007cost,gomez2010inferring}, document and corpus summarization 
\citep{lin2011class,kirchhoff2014submodularity,sipos2012temporal}, crowd 
teaching \citep{singla2014near}, feature selection 
\citep{elenberg2016restricted}, and interpreting deep neural networks 
\citep{elenberg2017streaming}. However, submodularity  goes  beyond
set functions and can be extended to continuous domains 
\citep{wolsey1982analysis,topkis1978minimizing}. 
Maximizing a submodular set function is inherently related to its continuous 
relaxation through the multilinear extension \citep{calinescu2011maximizing}, 
which is an example of the DR-submodular function. A variant of the Frank-Wolfe 
algorithm, called continuous 
greedy~\citep{calinescu2011maximizing,vondrak2008optimal},
can be used to
maximize, within a $(1-1/e)$ approximation to the optimum, the multilinear 
extension of a submodular set function~\citep{calinescu2011maximizing}
or more generally a monotone smooth submodular function subject to a 
polytope~\citep{chekuri2015multiplicative}.
It is also known that finding  a better approximation guarantee is impossible 
under reasonable complexity-theoretic 
assumptions~\citep{feige1998threshold,vondrak2013symmetry}.
More recently,  \citet{bian16guaranteed} generalized the above results by 
considering the maximization of continuous DR-submodular functions subject to 
down-closed convex bodies and showed that the same continuous greedy method 
achieves a $(1-1/e)$ guarantee.  In a different line of work, 
\citet{hassani2017gradient} studied the applicability of the (stochastic) 
gradient ascent algorithms to the \emph{stochastic} continuous submodular 
maximization setting, where the objective function is defined in terms of an 
expectation. They proved that gradient methods achieve a $1/2$ approximation 
guarantee for monotone DR-submodular functions, subject to a general convex 
body. It is also known that gradient methods cannot achieve a better guarantee 
in general \citep{hassani2017gradient, vondrak2011submodular}. Furthermore, it 
is also shown in \citep{hassani2017gradient} that the continuous greedy 
algorithms are not robust in stochastic settings (where only  unbiased 
estimates of gradients are available) and can provide arbitrarily poor 
solutions, in general (thus motivating the need for stochastic projected 
gradient methods). 
Even though it is not the focus of this paper, 
 we should mention that continuous submodular minimization has also been 
 studied recently~\citep{bach2015submodular, staib2017robust}. 

\paragraph{Online  optimization.} Most of the work in online optimization 
considers  convex (when minimizing the loss)  or concave (when maximizing the 
reward) functions. The protocol of online convex optimization (OCO) was first 
defined  by \citet{zinkevich2003online}. In his influential paper, he proposed 
the online gradient descent method and showed an $O(\sqrt{T})$ regret bound. 
The 
result was later improved to $O(\log(T))$  regret by 
\citet{hazan2007logarithmic} for strongly convex functions. 
\citet{kalai2005efficient} developed  another class of algorithms termed 
Follow-The-Leader (FTL) with the idea of finding a point that minimizes the 
accumulated sum of all objective functions revealed so far. However, there are 
simple situations in which the regret of FTL grows linearly with $T$. To 
circumvent this issue, \citet{kalai2005efficient} introduced random 
perturbation as a
regularization and proposed the follow-the-perturbed-leader algorithm, 
following an early work~\citep{hannan1957approximation}. In addition,
\citet{shalev2007primal} and 
\citet{abernethy2008competing} designed the regularized-follow-the-leader 
(RFTL) algorithm. A comprehensive survey of  OCO can be found in 
\citep{hazan2016introduction,shalev2012online}. Recently, 
\citet{lafond2015online} studied the setting in which the loss functions $ 
\{f_t:1\leq t\leq T \} $ are drawn i.i.d.\ from a fixed distribution and 
proposed the online Frank-Wolfe algorithm. They showed an $ O(\log^3 (T))$ 
regret for strongly convex loss functions. Furthermore, they showed that their 
algorithm finds a stationary point to the stochastic loss at a rate of $ 
O(\sqrt{1/T}) $. 
\citet{garber2013linearly} proposed a conditional gradient algorithm for online 
convex optimization problem over polyhedral sets. Only a single linear 
optimization step is performed in each iteration and this algorithm achieves 
$O(\sqrt{T})$ regret bound for convex losses and $O(\log T)$ regret bound for 
strongly convex losses. 
\citet{luo2014drifting} proposed a general methodology for devising online 
learning algorithms based on a
 drifting-games
analysis.
\citet{hazan2017efficient} goes beyond convexity and considered regret minimization in repeated games with non-convex loss functions. They introduced a new objective termed local regret and proposed online non-convex optimization algorithms that achieve optimal guarantees for this new objective. Our work, in contrast, considers non-convex objective functions that can be approximately maximized. In our notion of $\alpha$-regret, we design two algorithms that can compete with the best fixed offline approximate solution (and not necessarily the stationary points) with tight regret bounds.  

\paragraph{Online submodular optimization.}  
Existing work considered online submodular optimization in a discrete domain. 
\citet{streeter2009online} and \citet{golovin14online} proposed online optimization algorithms for submodular set functions under cardinality and matroid constraints, respectively. 
Our work studies the online submodular optimization in continuous domains. We 
should point out that the online algorithm proposed in \citet{golovin14online} 
relies on the multilinear continuous relaxation, which is simply an instance 
of 
the general class of DR-submodular functions that we consider here.    

\section{Conclusion}
In this paper, we considered an online optimization process, where the 
objective functions were continuous DR-submodular. We proposed two online 
optimization algorithms, \OFW (that has access to exact gradients) and \OGA 
(that 
only has  access to unbiased estimates of the gradients), both with no-regret 
guarantees. We also evaluated the performance of our algorithms in practice. 
Our results make an important contribution in providing performance guarantees 
for a subclass of online non-convex optimization problems. 

\subsubsection*{Acknowledgments}
This work was supported by AFOSR YIP award (FA9550-18-1-0160).


\clearpage
\bibliographystyle{plainnat}
\bibliography{reference-list}
\clearpage
\appendix
\onecolumn
\begin{appendices}
	\crefalias{section}{appsec}
\section{Proof of \cref{thm:guarantee-frank-wolfe}}
\label{sec:proof-frank-wolfe}
Before presenting the proof of \cref{thm:guarantee-frank-wolfe}, we need two 
lemmas first. 
\cref{lem:beta-smooth} shows that a $ \beta $-smooth function can be 
bounded by quadratic functions from above and below. \cref{lem:concave} shows 
the concavity of continuous DR-submodular functions along non-negative and 
non-positive directions.

\begin{lem}
	If $ f $ is $ \beta $-smooth, then we have for any $ \mathbf{x} $ and $ \mathbf{y} $, \[ 
	\left| f(\mathbf{x})-f(\mathbf{y})-\nabla f(\mathbf{x})^{\top} (\mathbf{x}-\mathbf{y})  \right| \leq \frac{\beta}{2} \lVert \mathbf{x}-\mathbf{y} \rVert^2.
	\]
	\[ 
	\left| f(\mathbf{x})-f(\mathbf{y})-\nabla f(\mathbf{y})^{\top} (\mathbf{x}-\mathbf{y})  \right| \leq \frac{\beta}{2} \lVert \mathbf{x}-\mathbf{y} \rVert^2.
	\]
	\label{lem:beta-smooth}
\end{lem}
\begin{proof}
	Let us define an auxiliary function $ g(t) = f(\mathbf{x}+ t(\mathbf{y}-\mathbf{x}) ) $. We observe that $ g(0)=f(\mathbf{x}) $ and $ g(1)=f(\mathbf{y}) $. The derivative of $ g(t) $ is 
	\[ g'(t) = \nabla f(\mathbf{x}+ t(\mathbf{y}-\mathbf{x}))^{\top} (\mathbf{y}-\mathbf{x}). \]
	We have 
	\begin{equation*}
		f(\mathbf{y})-f(\mathbf{x})=g(1)-g(0)=\int_{0}^{1} g'(t) dt = \int_{0}^{1} \nabla f(\mathbf{x}+ t(\mathbf{y}-\mathbf{x}))^{\top} (\mathbf{y}-\mathbf{x}) dt.
	\end{equation*}
	The left-hand side of the first inequality is equal to
	\begin{dmath*}
		\left| \int_{0}^{1} \nabla f(\mathbf{x}+ t(\mathbf{y}-\mathbf{x}))^{\top} (\mathbf{x}-\mathbf{y}) dt-\nabla f(\mathbf{x})^{\top} (\mathbf{x}-\mathbf{y})  \right| 
		= \left| \int_{0}^{1}( \nabla f(\mathbf{x}+ t(\mathbf{y}-\mathbf{x})) - \nabla f(\mathbf{x}) )^{\top} (\mathbf{x}-\mathbf{y})  \right| dt
		\leq \int_{0}^{1} \left| ( \nabla f(\mathbf{x}+ t(\mathbf{y}-\mathbf{x})) - \nabla f(\mathbf{x}) )^{\top} (\mathbf{x}-\mathbf{y})   \right| dt
		\leq \int_{0}^{1} \lVert \nabla f(\mathbf{x}+ t(\mathbf{y}-\mathbf{x})) - \nabla f(\mathbf{x}) \rVert \lVert \mathbf{x}-\mathbf{y} \rVert dt
		\leq \int_{0}^{1} \beta t \lVert \mathbf{x}-\mathbf{y} \rVert^2 dt
		= \frac{\beta}{2} \lVert \mathbf{x}-\mathbf{y} \rVert.
	\end{dmath*}
	Exchanging $ \mathbf{x} $ and $ \mathbf{y} $ in the first inequality, we obtain the second one immediately.
\end{proof}

\begin{lem}[Proposition~4 in \cite{bian16guaranteed}]
	\label{lem:concave}
	A continuous DR-submodular function is concave along any non-negative
	direction and any non-positive direction.
\end{lem}

\cref{lem:concave} implies that if $ f $ is continuous DR-submodular, fixing 
any $ x $ in its domain, $ g(z)\triangleq f(\mathbf{x}+z\mathbf{v}) $ is 
concave in $ z $ as long as $ \mathbf{v}\geq 0 $ holds elementwise.
Now we present the proof of \cref{thm:guarantee-frank-wolfe}.
\begin{proof}
As the first step, let us fix $ t $ and $ k $.	Since $ f_t $ is $ \beta 
$-smooth, by Lemma~\ref{lem:beta-smooth}, 
	for any $\xi\geq0$ and ${\bf x},{\bf v}\in\mathbb{R}_{\geq0}^{n}$,
	we have 
	\begin{dmath*}
		f_t(\mathbf{x}+\xi \mathbf{v}) -f_t(\mathbf{x})  - \nabla f_t(\mathbf{x})^{\top} (\xi \mathbf{v}) \geq - \frac{\beta}{2} \lVert \xi \mathbf{v} \rVert^2
	\end{dmath*}
	Let $ L \triangleq \beta R^2 $. We deduce
	\begin{equation*}
	 f_{t}({\bf x}_{t}(k+1))-f_{t}({\bf x}_{t}(k))
	=  f_{t}({\bf x}_{t}(k)+\frac{1}{K}{\bf v}_{t}^{k})-f_{t}({\bf x}_{t}(k))
	\geq  \frac{1}{K}\langle{\bf v}_{t}^{k},\nabla f_{t}({\bf 
	x}_{t}(k))\rangle-\frac{L}{2K^2}. 
	\end{equation*}
	We sum the above equation over $t$ and obtain
	\begin{dmath*}
		\sum_{t=1}^{T}f_{t}({\bf x}_{t}(k+1))-f_{t}({\bf 
		x}_{t}(k))\geq\sum_{t=1}^{T}\frac{1}{K}\langle{\bf v}_{t}^{k},\nabla 
		f_{t}({\bf x}_{t}(k))\rangle-\frac{LT}{2K^2}. 
	\end{dmath*}
	The RFTL algorithm instance $\mathcal{E}^{k}$ finds $\{{\bf v}_{t}^{k}:1\leq t\leq T\}$
	such that
	\begin{equation*}
	 \sum_{t=1}^{T}\langle{\bf v}^{k*},\nabla f_{t}({\bf x}_{t}(k))\rangle-\sum_{t=1}^{T}\langle{\bf v}_{t}^{k},\nabla f_{t}({\bf x}_{t}(k))\rangle
	\leq  r^{k}\leq 2DG\sqrt{T},
	\end{equation*}
	where
	\[
	{\bf v}^{k*}=\argmax_{{\bf v}\in\mathcal{P}}\sum_{t=1}^{T}\langle{\bf v},\nabla f_{t}({\bf x}_{t}(k))\rangle
	\]
	and $r^{k}$ is the total regret that the RFTL instance suffers by
	the end of the $T$th iteration. According to the regret bound of
	the RFTL, we know that $r^{k}\leq2DG\sqrt{T}$. Therefore,
	\begin{equation*}
	 \sum_{t=1}^{T}f_{t}({\bf x}_{t}(k+1))-f_{t}({\bf x}_{t}(k))
	\geq  \frac{1}{K}\left(\sum_{t=1}^{T}\langle{\bf v}^{k*},\nabla f_{t}({\bf 
	x}_{t}(k))\rangle-r^{k}\right)-\frac{LT}{2K^2}. 
	\end{equation*}
	We define ${\bf x}^{*}\triangleq\argmax_{{\bf v}\in\mathcal{P}}\sum_{t=1}^{T}f_{t}({\bf v})$ and
	${\bf w}_{t}^{k}=({\bf x}^{*}-{\bf x}_{t}(k))\vee0$.
	For every $t$, we have ${\bf w}_{t}^{k}=({\bf x}^{*}-{\bf x}_{t}(k))\vee0\leq{\bf x}^{*}$.
    It is obvious
	that ${\bf w}_{t}^{k}
    \geq0$. 
    Therefore we deduce that ${\bf w}_{t}^{k}\in\mathcal{X}$.
	Due to the concavity of $f_{t}$ along any non-negative direction (see 
	\cref{lem:concave}),
	we have
	\[
	f_{t}({\bf x}_{t}(k)+{\bf w}_{t}^{k})-f_{t}({\bf x}_{t}(k))\leq\langle{\bf w}_{t}^{k},\nabla f_{t}({\bf x}_{t}(k))\rangle.
	\]
	In light of the above equation, we obtain a lower bound for 
	$\sum_{t=1}^{T}\langle{\bf v}^{k*},\nabla f_{t}({\bf x}_{t}(k))\rangle$:
	\begin{align*}
	 \sum_{t=1}^{T}\langle{\bf v}^{k*},\nabla f_{t}({\bf x}_{t}(k))\rangle
	\geq & \sum_{t=1}^{T}\langle
    \mathbf{x}^{*}
    ,\nabla f_{t}({\bf x}_{t}(k))\rangle\\
	\stackrel{(a)}{\geq} & \sum_{t=1}^{T}\langle{\bf w}_{t}^{k},\nabla f_{t}({\bf x}_{t}(k))\rangle\\
	\geq & \sum_{t=1}^{T}(f_{t}({\bf x}_{t}(k)+{\bf w}_{t}^{k})-f_{t}({\bf x}_{t}(k)))\\
	= & \sum_{t=1}^{T}(f_{t}({\bf x}^{*}\vee{\bf x}_{t}(k))-f_{t}({\bf x}_{t}(k)))\\
	\geq & \sum_{t=1}^{T}(f_{t}({\bf x}^{*})-f_{t}({\bf x}_{t}(k))).
	\end{align*}
	We use the fact that $\nabla f_{t}({\bf x}_{t}(k))\geq0$ and $\mathbf{x}^*\geq \mathbf{w}_t^k$ entrywise in the inequality
	(a).
	\begin{equation*}
	 \sum_{t=1}^{T}f_{t}({\bf x}_{t}(k+1))-f_{t}({\bf x}_{t}(k))
	\geq  \frac{1}{K}\left(\sum_{t=1}^{T}(f_{t}({\bf x}^{*})-f_{t}({\bf 
	x}_{t}(k)))-r^{k}\right)-\frac{LT}{2K^2}. 
	\end{equation*}
	After rearrangement,
	\begin{equation*}
	 \sum_{t=1}^{T}\left(f_{t}({\bf x}_{t}(k+1))-f_{t}({\bf x}^{*})\right)
	\geq  (1-\frac{1}{K})\sum_{t=1}^{T}\left(f_{t}({\bf x}_{t}(k))-f_{t}({\bf 
	x}^{*})\right)-\frac{1}{K}r^{k}-\frac{LT}{2K^2} .
	\end{equation*}
	Therefore,
	\begin{dmath*}
	 \sum_{t=1}^{T}\left(f_{t}({\bf x}_{t}(K))-f_{t}({\bf x}^{*})\right)
	\geq  
	(1-\frac{1}{K})^K \sum_{t=1}^{T}\left(f_{t}({\bf 
	x}_{t}(0))-f_{t}({\bf x}^{*})\right) 
	-\frac{1}{K}\sum_{k=0}^{K-1}r^{k}-\frac{LT}{2K} 
	= 
	(1-\frac{1}{K})^K\sum_{t=1}^{T}\left(f_{t}(0)-f_{t}({\bf 
	x}^{*})\right)-\frac{1}{K}\sum_{k=0}^{K-1}r^{k}-\frac{LT}{2K}. 
	\end{dmath*}
	
	Since $
	(1-\frac{1}{K})^K\leq 
	e^{-1}$, we have
	\begin{dmath*}
	 \sum_{t=1}^{T}\left(f_{t}({\bf x}^{*})-f_{t}({\bf x}_{t}(K))\right)
	\leq  
	(1-\frac{1}{K})^K\sum_{t=1}^{T}\left(f_{t}({\bf 
	x}^{*})-f_{t}(0)\right) 
	+\frac{1}{K}\sum_{k=0}^{K-1}r^{k}+\frac{LT}{2K} 
	\leq  e^{-1}\sum_{t=1}^{T}\left(f_{t}({\bf x}^{*})-f_{t}(0)\right) 
	+\frac{1}{K}\sum_{k=0}^{K-1}r^{k}+\frac{LT}{2K}. 
	\end{dmath*}
	After rearrangement, we have
	\begin{dmath*}
		\sum_{t=1}^{T}f_{t}({\bf x}_{t}(K)) 
		\geq  (1-1/e)\sum_{t=1}^{T}f_{t}({\bf x}^{*})+e^{-1}\sum_{t=1}^{T}f_{t}(0)
		-\frac{1}{K}\sum_{k=0}^{K-1}r^{k}-\frac{LT}{2K^2}. 
	\end{dmath*}
Plugging in the definition of $ r^k $ gives
	\begin{equation*}
		\sum_{t=1}^{T}f_{t}({\bf x}_{t})=\sum_{t=1}^{T}f_{t}({\bf x}_{t}(K))\geq  (1-1/e)\sum_{t=1}^{T}f_{t}({\bf x}^{*})+e^{-1}\sum_{t=1}^{T}f_{t}(0)-2DG\sqrt{T}-\frac{LT}{2K}.
	\end{equation*}
	Recall that $ \mathbf{x}_t(K) $ is exactly $ \mathbf{x}_t $. Thus equivalently, we have
	\begin{dmath*}
		(1-1/e)\sum_{t=1}^{T}f_{t}({\bf x}^{*})-\sum_{t=1}^{T}f_{t}({\bf x}_{t})\leq-e^{-1}\sum_{t=1}^{T}f_{t}(0)+2DG\sqrt{T}+\frac{ \beta R^2 T}{2K}.
	\end{dmath*}
\end{proof}

\section{Proof of \cref{thm:gradient-ascent}}\label{sec:proof-gradient}
\subsection{Gradient Ascent Case}
The theoretical guarantee of gradient ascent methods applied to concave functions relies on a pivotal property that characterizes concavity: if $ F $ is concave, then $ F(\mathbf{y})-F(\mathbf{x}) \leq \langle \nabla F(\mathbf{x}), \mathbf{y}-\mathbf{x}  \rangle $.
Fortunately, there is a similar property that holds for monotone weakly DR-submodular functions, which is presented in \cref{lem:approximate-concave}.
\begin{lem}\label{lem:approximate-concave}
	Let $F:\mathcal{X}\to\mathbb{R}_{+}$ be a monotone
	and weakly DR-submodular function with parameter $\gamma$. For any
	two vector ${\bf x},{\bf y}\in\mathcal{X}$, we have
	\[
	F({\bf y})-\left(1+\frac{1}{\gamma^{2}}\right)F({\bf x})\leq\frac{1}{\gamma}\left\langle \nabla F({\bf x}),{\bf y}-{\bf x}\right\rangle .
	\]\label{lem:key-lemma}
\end{lem}

The proof of \cref{lem:key-lemma} can be found in the proof of Theorem~4.2 
in~\citep{hassani2017gradient}. Now we can prove \cref{thm:gradient-ascent} in 
the gradient ascent case.

\begin{proof}
	Let ${\bf x}^{*}=\argmax_{{\bf x}\in\mathcal{P}}\sum_{t=1}^{T}f_{t}({\bf x})$.
	We define $\nabla_{t}\triangleq\nabla f_{t}({\bf x}_{t})$. By the definition of $ \mathbf{x}_{t+1} $ and properties of the projection operator for a convex set, we have
		\begin{equation*}
	\left\Vert {\bf x}_{t+1}-{\bf x}^{*}\right\Vert ^{2}
	  = \left\Vert \Pi_{\mathcal{P}}({\bf x}_{t}+\eta_{t}\nabla_{t})-{\bf x}^{*}\right\Vert ^{2}
	 \leq \left\Vert {\bf x}_{t}+\eta_{t}\nabla_{t}-{\bf x}^{*}\right\Vert ^{2}
	 \leq  \left\Vert {\bf x}_{t}-{\bf x}^{*}\right\Vert ^{2}+\eta_{t}^{2}\left\Vert \nabla_{t}\right\Vert ^{2}-2\eta_{t}\nabla_{t}^{\top}({\bf x}^{*}-{\bf x}_{t}).
	\end{equation*}
	Therefore we deduce 
	\begin{align*}
	\nabla_{t}^{\top}({\bf x}^{*}-{\bf x}_{t}) & \leq\frac{\left\Vert {\bf x}_{t}-{\bf x}^{*}\right\Vert ^{2}-\left\Vert {\bf x}_{t+1}-{\bf x}^{*}\right\Vert ^{2}+\eta_{t}^{2}\left\Vert \nabla_{t}\right\Vert ^{2}}{2\eta_{t}}\\
	& \leq\frac{\left\Vert {\bf x}_{t}-{\bf x}^{*}\right\Vert ^{2}-\left\Vert {\bf x}_{t+1}-{\bf x}^{*}\right\Vert ^{2}}{2\eta_{t}}+\frac{\eta_{t}G^{2}}{2}
	\end{align*}
	By Lemma
	\ref{lem:key-lemma}, we obtain that
	\[
	f_{t}({\bf x}^{*})-\left(1+\frac{1}{\gamma^{2}}\right)f({\bf x}_{t})\leq\frac{1}{\gamma}\left\langle \nabla_{t},{\bf x}^{*}-{\bf x}_{t}\right\rangle .
	\]
	If we define $\frac{1}{\eta_{0}}\triangleq0$, it can be deduced that
	\begin{dmath*}
	 \sum_{t=1}^{T}\left[f_{t}({\bf x}^{*})-\left(1+\frac{1}{\gamma^{2}}\right)f_{t}({\bf x}_{t})\right]
	\leq  \frac{1}{\gamma}\sum_{t=1}^{T}\nabla_{t}^{\top}({\bf x}^{*}-{\bf x}_{t})
	\leq  \frac{1}{\gamma}\left[\frac{1}{2\eta_{t}}\sum_{t=1}^{T}\left(\left\Vert {\bf x}_{t}-{\bf x}^{*}\right\Vert ^{2}-\left\Vert {\bf x}_{t+1}-{\bf x}^{*}\right\Vert ^{2}\right)+\frac{G^{2}}{2}\sum_{t=1}^{T}\eta_{t}\right]
	\leq  \frac{1}{\gamma}\left[\frac{1}{2}\left(\sum_{t=1}^{T}\left\Vert {\bf x}_{t}-{\bf x}^{*}\right\Vert ^{2}(\frac{1}{\eta_{t}}-\frac{1}{\eta_{t-1}})\right)+\frac{G^{2}}{2}\sum_{t=1}^{T}\eta_{t}\right]
	\leq  \frac{1}{\gamma}\left[\frac{D^{2}}{2\eta_{T}}+\frac{G^{2}}{2}\sum_{t=1}^{T}\eta_{t}\right]
	\leq  \frac{3}{2\gamma}DG\sqrt{T}.
	\end{dmath*}
	After rearrangement, it is clear that
	\[
	\frac{\gamma^{2}}{\gamma^{2}+1}\sum_{t=1}^{T}f_{t}({\bf x}^{*})-\sum_{t=1}^{T}f_{t}({\bf x}_{t})\leq\frac{3\gamma DG\sqrt{T}}{2(\gamma^{2}+1)}.
	\]
\end{proof}

\subsection{Stochastic Gradient Ascent Case}

\begin{proof}
	The strategy for the stochastic gradient ascent case is similar to that of the gradient ascent case.
	Again, by the definition of $ \mathbf{x}_{t+1} $,
	we have
	\begin{equation*}
	\left\Vert {\bf x}_{t+1}-{\bf x}^{*}\right\Vert ^{2}
	  = \left\Vert \Pi_{\mathcal{P}}({\bf x}_{t}+\eta_{t}\mathbf{g}_{t})-{\bf x}^{*}\right\Vert ^{2}
	 \leq \left\Vert {\bf x}_{t}+\eta_{t}\mathbf{g}_{t}-{\bf x}^{*}\right\Vert ^{2}
	 \leq \left\Vert {\bf x}_{t}-{\bf x}^{*}\right\Vert ^{2}+\eta_{t}^{2}\left\Vert {\bf g}_{t}\right\Vert ^{2}-2\eta_{t}{\bf g}_{t}^{\top}({\bf x}^{*}-{\bf x}_{t})
	\end{equation*}
	Therefore we deduce
	\begin{equation*}
	{\bf g}_{t}^{\top}({\bf x}^{*}-{\bf x}_{t}) \leq\frac{\left\Vert {\bf x}_{t}-{\bf x}^{*}\right\Vert ^{2}-\left\Vert {\bf x}_{t+1}-{\bf x}^{*}\right\Vert ^{2}+\eta_{t}^{2}\left\Vert {\bf g}_{t}\right\Vert ^{2}}{2\eta_{t}} \leq\frac{\left\Vert {\bf x}_{t}-{\bf x}^{*}\right\Vert ^{2}-\left\Vert {\bf x}_{t+1}-{\bf x}^{*}\right\Vert ^{2}}{2\eta_{t}}+\frac{\eta_{t}G^{2}}{2}
	\end{equation*}
	Similarly, if we define $\frac{1}{\eta_{0}}\triangleq0$ and in light of \cref{lem:approximate-concave}, it can be deduced that
	\begin{dmath*}
	 \sum_{t=1}^{T}\mathbb{E}\left[f_{t}({\bf x}^{*})-\left(1+\frac{1}{\gamma^{2}}\right)f_{t}({\bf x}_{t})\right]
	\leq  \frac{1}{\gamma}\sum_{t=1}^{T}\mathbb{E}\left[\nabla_{t}^{\top}({\bf x}^{*}-{\bf x}_{t})\right]
	=  \frac{1}{\gamma}\sum_{t=1}^{T}\mathbb{E}\left[\mathbb{E}\left[\nabla_{t}^{\top}({\bf x}^{*}-{\bf x}_{t})|{\bf x}_{t}\right]\right]
	=  \frac{1}{\gamma}\sum_{t=1}^{T}\mathbb{E}\left[\mathbb{E}\left[{\bf g}_{t}^{\top}({\bf x}^{*}-{\bf x}_{t})|{\bf x}_{t}\right]\right]
	\leq  \frac{1}{\gamma}\left[\frac{1}{2\eta_{t}}\sum_{t=1}^{T}\mathbb{E}\left[\left\Vert {\bf x}_{t}-{\bf x}^{*}\right\Vert ^{2}-\left\Vert {\bf x}_{t+1}-{\bf x}^{*}\right\Vert ^{2}\right]+
    \frac{G^{2}}{2}\sum_{t=1}^{T}\eta_{t}\right]
	\leq  \frac{1}{\gamma}\left[\frac{1}{2}\left(\sum_{t=1}^{T}\mathbb{E}\left[\left\Vert {\bf x}_{t}-{\bf x}^{*}\right\Vert ^{2}\right](\frac{1}{\eta_{t}}-\frac{1}{\eta_{t-1}})\right)+\frac{G^{2}}{2}\sum_{t=1}^{T}\eta_{t}\right]
	\leq  \frac{1}{\gamma}\left[\frac{D^{2}}{2\eta_{T}}+\frac{G^{2}}{2}\sum_{t=1}^{T}\eta_{t}\right]
	\leq  \frac{3}{2\gamma}DG\sqrt{T}.
	\end{dmath*}
	After rearrangement, it is clear that
	\[
	\frac{\gamma^{2}}{\gamma^{2}+1}\sum_{t=1}^{T}f_{t}({\bf x}^{*})-\sum_{t=1}^{T}\mathbb{E}\left[f_{t}({\bf x}_{t})\right]\leq\frac{3\gamma DG\sqrt{T}}{2(\gamma^{2}+1)}.
	\]
\end{proof}
\end{appendices}
\end{document}